\documentclass[letterpaper]{article} % DO NOT CHANGE THIS
\usepackage{aaai2026}  % DO NOT CHANGE THIS
\usepackage{times}  % DO NOT CHANGE THIS
\usepackage{helvet}  % DO NOT CHANGE THIS
\usepackage{courier}  % DO NOT CHANGE THIS
\usepackage[hyphens]{url}  % DO NOT CHANGE THIS
\usepackage{graphicx} % DO NOT CHANGE THI
\usepackage{enumitem} % For customising lists
\usepackage{natbib}  % DO NOT CHANGE THIS AND DO NOT ADD ANY OPTIONS TO IT
\usepackage{caption} % DO NOT CHANGE THIS AND DO NOT ADD ANY OPTIONS TO IT
\usepackage{tcolorbox}
\usepackage{algorithm}
\usepackage{algorithmic}
\tcbuselibrary{skins}
\newtcolorbox{theorembox}[1][]{
  enhanced,
  arc=2mm,
  borderline west={1mm}{0pt}{gray},
  colback=gray!10,
  colframe=gray!50,
  fonttitle=\bfseries,
  title=#1
}

\usepackage{newfloat}
\usepackage{listings}

\usepackage{graphicx}%
\usepackage{multirow}%
\usepackage{amsmath,amssymb,amsfonts}%
\usepackage{amsthm}%
\usepackage{mathrsfs}%
\usepackage[title]{appendix}%
\usepackage{xcolor}%
\usepackage{textcomp}%
\usepackage{manyfoot}%
\usepackage{booktabs}%
\usepackage{algorithm}%
\usepackage{algorithmic}%
\usepackage{listings}%

\usepackage{subcaption} 

\usepackage{tikz}
\usepackage{tikz-3dplot}
\usetikzlibrary{arrows.meta, backgrounds, fadings}
\newtheorem{definition}{Definition}

\newtheorem{theorem}{Theorem}
\newtheorem{lemma}{Lemma}

\newtheorem{proposition}{Proposition}
\newtheorem{assumption}{Assumption}

\DeclareCaptionStyle{ruled}{labelfont=normalfont,labelsep=colon,strut=off} % DO NOT CHANGE THIS
\floatstyle{ruled}
\newfloat{listing}{tb}{lst}{}
\floatname{listing}{Listing}
\title{Partial Action Replacement: Tackling Distribution Shift in Offline MARL}

\author{
    Yue Jin\textsuperscript{\rm 1},
    Giovanni Montana\textsuperscript{\rm 1, 2, 3}
}
\affiliations{
    \textsuperscript{\rm 1} Warwick Manufacturing Group, University of Warwick, Coventry, CV4 7AL, UK \\
    \textsuperscript{\rm 2} Department of Statistics, University of Warwick, Coventry, CV4 7AL, UK \\
    \textsuperscript{\rm 3} The Alan Turing Institute, 96 Euston Road, London, NW1 2DB, UK \\
    \{yue.jin.3, g.montana\}@warwick.ac.uk
}

\usepackage{bibentry}
\begin{document}

\maketitle

\begin{abstract}
Offline multi-agent reinforcement learning (MARL) is severely hampered by the challenge of evaluating out-of-distribution (OOD) joint actions. Our core finding is that when the behavior policy is factorized—a common scenario where agents act fully or partially independently during data collection—a strategy of partial action replacement (PAR) can significantly mitigate this challenge. PAR updates a single or part of agents' actions while the others remain fixed to the behavioral data, reducing distribution shift compared to full joint-action updates. Based on this insight, we develop Soft-Partial Conservative Q-Learning (SPaCQL), using PAR to mitigate OOD issue and dynamically weighting different PAR strategies based on the uncertainty of value estimation. We provide a rigorous theoretical foundation for this approach, proving that under factorized behavior policies, the induced distribution shift scales linearly with the number of deviating agents rather than exponentially with the joint-action space. This yields a provably tighter value error bound for this important class of offline MARL problems. Our theoretical results also indicate that SPaCQL adaptively addresses distribution shift using uncertainty-informed weights. Our empirical results demonstrate SPaCQL enables more effective policy learning, and manifest its remarkable superiority over baseline algorithms when the offline dataset exhibits the independence structure.
\end{abstract}


%%%%%%%%%%%%%%%%%%%%%%%%%%%%%%%%

\section{Introduction}

Multi-agent reinforcement learning (MARL) offers a powerful paradigm for solving complex coordination problems, yet its real-world application is often constrained by the need for extensive online environment interaction \cite{Gronauer2022MultiSurvey, Wong2022DeepDirections, Dinneweth2022Multi-agentSurvey}. Offline MARL overcomes this by enabling agents to learn from fixed, pre-collected datasets. This approach, however, confronts a fundamental obstacle rooted in the combinatorial nature of multi-agent systems: the curse of dimensionality in the joint-action space. Any finite dataset provides only sparse coverage of all possible action combinations. This forces the learning algorithm to reason about the value of countless out-of-distribution (OOD) joint actions. Standard value-based methods, like Q-learning, are notoriously prone to failure in this regime \cite{Kumar2020, Kumar2019StabilizingReduction}. A function approximator, such as a neural network, can produce arbitrarily high Q-values for these unseen OOD actions. These erroneous estimates guide agents toward divergent policies that fail to generalize, making robust offline MARL a significant challenge \cite{Yang2021BelieveLearning, Shao2023CounterfactualLearning}.

This OOD problem can be understood from a geometric perspective. The offline dataset represents a sparse collection of known joint actions within the vast, high-dimensional space of all possible action combinations. A conventional Q-learning update requires evaluating a joint action where all agents adopt their new, learned policies. This is akin to querying a point far from any known data, forcing the Q-function to perform a large and unreliable extrapolation into uncharted regions of the joint-action space. The core insight of our work is that we can avoid such risky extrapolations by instead querying points that remain anchored close to the known data.

We operationalize this insight through a principle we term \emph{partial action replacement}: changing only a single or part of agents' actions while keeping the others fixed to actions from the dataset. This simple change ensures the Q-value is queried for a joint action that differs minimally from the known data—requiring only a small, local extrapolation rather than a large leap into the unknown. The intuition is straightforward: if we change one coordinate of a known data point, we stay much closer to familiar territory than if we change all coordinates simultaneously.

Crucially, this approach is most effective when the offline dataset was collected by agents acting independently or whose behaviors are loosely coordinated—a common scenario in practice, including independent human demonstrations, independently trained agents, or decentralized systems \cite{Chen2017DecentralizedLearning, Tampuu2017MultiagentLearning, Leibo2017Multi-agentDilemmas}. Under this factorized behavior policy assumption, we can prove that while the joint-action space remains exponentially large, our approach of partial replacement achieves linear scaling of distribution shift and value estimation error with respect to the number of deviating agents. This theoretical insight directly motivates our algorithmic design.

Based on this principle, we develop algorithms that navigate the critical balance between estimation stability and inter-agent coordination. We first present Individual CQL with Q-sharing (ICQL-QS), which directly implements partial replacement by changing only one agent's action at a time to maximize stability. While highly stable, this conservative approach may miss valuable coordinated behaviors. Our primary contribution, Soft Partial Conservative Q-Learning (SPaCQL), addresses this limitation by forming an adaptive algorithm that dynamically weights different coordination combinations based on uncertainty of target values, as measured by the standard deviation of a Q-function ensemble \cite{An2021Uncertainty-BasedQ-Ensemble, Osband2016DeepDQN}.

We provide a rigorous theoretical foundation for this approach. Our analysis formally proves that under factorized behavior policies, the distribution shift induced by partial action replacement scales linearly with the number of deviating agents, in stark contrast to the exponential complexity of arbitrary joint-action updates. We leverage this result to derive a novel, MARL-specific value-error bound for this important class of problems and then prove that SPaCQL adjusts distribution shift and value-error bound adaptively based on uncertainty-informed weights.

The main contributions of this paper are therefore:
\begin{itemize}
    \item We introduce and formalize the principle of partial action replacement, showing that under factorized behavior policies, it reduces distribution shift in a multi-agent setting.
    
    \item We develop SPaCQL, an algorithm that dynamically weights partial action replacement strategies based on uncertainty estimates, adapting between conservative single-agent updates and coordinated multi-agent updates based on data characteristics.
    
    \item We provide theoretical analysis including a novel value-error bound that scales linearly with the number of deviating agents rather than exponentially with the joint-action space, applicable when the behavior policy is factorized.
    
    \item Empirical results demonstrate SPaCQL's efficacy, especially its consistent and remarkable superiority on all Random and Medium-Replay datasets where agent behaviors are less coordinated.
\end{itemize}

%%%%%%%%%%%%%%%%%%%%%%%%%%%%%%%%%%%%%%%%

\section{Related Work}

Offline reinforcement learning addresses learning effective policies from static datasets without environment interaction. Single-agent methods like CQL \cite{Kumar2020}, IQL \cite{Kostrikov2022OfflineQ-Learning}, and BEAR \cite{Kumar2019StabilizingReduction} tackle distribution shift by constraining policies or regularizing value functions. The multi-agent setting amplifies these challenges due to exponential joint-action space growth and coordination requirements. Offline MARL research focuses on two main approaches: constraining learned policies or regularizing value functions. Our work builds on the latter with a novel adaptive mechanism.

\paragraph{Policy-Constrained Methods}
Building on single-agent approaches like AWR \cite{Peng2019Advantage-WeightedLearning} and AWAC \cite{Nair2020AWAC:Datasets}, these methods ensure policies don't deviate excessively from the behavior policy \cite{Pan2022PlanRectification, Tseng2022OfflineDistillation}. \cite{Pan2022PlanRectification} use evolution strategies for decentralized regularization, while \cite{Tseng2022OfflineDistillation} employs a Teacher-Student paradigm where a centralized transformer predicts joint actions and individual policies mimic both actions and structural relationships.

\paragraph{Value-Constrained Methods}
These methods constrain the value function, penalizing OOD actions \cite{Yang2021BelieveLearning, Shao2023CounterfactualLearning, Barde2024AProblem, Ma2023LearningPolicies, Wang2023OfflineFactorization, Wang2023OfflineRegularization}. 

CFCQL \cite{Shao2023CounterfactualLearning} extends CQL \cite{Kumar2020} to multi-agent settings, penalizing each agent's OOD action individually while holding others constant—sharing motivation with our partial replacement. However, we use partial replacement for target value computation (not just regularization) and introduce adaptive weighting based on uncertainty.

Others adapt online MARL mechanisms like QMIX \cite{Rashid2018QMIX:Learning} and QTRAN \cite{Son2019QTRAN:Learning} to offline settings. \cite{Wang2023OfflineFactorization} and \cite{Wang2023OfflineRegularization} decompose global Q-functions: the former uses IQL's expectile regression, the latter formulates convex optimization. We differ by using centralized training with shared Q-functions, capturing coordination implicitly.

\cite{Barde2024AProblem} use world models to generate trajectories with uncertainty-modified rewards. Recent works employ diffusion models: DoF \cite{Li2025DOF:LEARNING}, MADIFF \cite{Zhu2024MADIFF:Models}, and INS \cite{Fu2025INS:LEARNING}.

Uncertainty estimation via ensembles \cite{An2021Uncertainty-BasedQ-Ensemble, Bai2022PessimisticLearning, Zhang2020RobustUncertainty, Osband2016DeepDQN, Lakshminarayanan2017SimpleEnsembles} helps avoid overestimating OOD actions. We use ensemble variance to weight partial replacement strategies—higher uncertainty reduces contribution in Q-updates, adapting to dataset characteristics.

\section{Preliminaries}

We study offline MARL in a Decentralized Markov Decision Process (Dec-MDP) $\langle \mathcal{S}, \{\mathcal{A}_i\}_{i=1}^n, P, R, n, \gamma \rangle$, where $n$ agents interact in state space $\mathcal{S}$. At each timestep, agent $i$ selects action $a_{t,i} \in \mathcal{A}_i$, forming joint action $\boldsymbol{a}_t \in \boldsymbol{\mathcal{A}} = \times_{i=1}^n \mathcal{A}_i$. The environment transitions to $s_{t+1} \sim P(\cdot|s_t, \boldsymbol{a}_t)$ and yields reward $r_t = R(s_t, \boldsymbol{a}_t)$. 

In the offline setting, agents cannot interact with the environment. Learning occurs solely from a static dataset $\mathcal{D} = \{(s, \boldsymbol{a}, r, s')_k\}$ collected by behavior policy $\boldsymbol{\mu}(\boldsymbol{a}|s)$. The goal remains learning a joint policy $\boldsymbol{\pi}(\boldsymbol{a}|s)$ that maximizes expected return $J(\boldsymbol{\pi}) = \mathbb{E}[\sum_{t=0}^{\infty} \gamma^t r_t]$, with Q-function defined as:
{\footnotesize
$$Q^{\boldsymbol{\pi}}(s, \boldsymbol{a}) = \mathbb{E}_{\boldsymbol{\pi}} \left[ \sum_{t=0}^{\infty} \gamma^t R(s_t, \boldsymbol{a}_t) \, \middle| \, s_0=s, \boldsymbol{a}_0=\boldsymbol{a} \right].$$
}
The core challenge is that distribution shift between $\boldsymbol{\pi}$ and $\boldsymbol{\mu}$ causes Q-learning to overestimate OOD joint actions, leading to policy divergence.

\section{Methodology}\label{sec:meth}

Our approach is founded on a simple insight for mitigating the challenge of OOD actions in offline MARL. The core problem is that a Q-function trained on a static dataset can produce arbitrarily wrong values for novel joint actions not seen in the data. A conventional update, which evaluates a joint action where all agents follow their new policies, is almost certain to be OOD. Our guiding principle is to avoid such large, unreliable extrapolations. Instead, we propose to evaluate a \emph{partially new} joint action that differs from the dataset by only one or a few individual actions. This small, anchored step is inherently more stable and leads to more reliable Q-value estimates. 

%%%

\subsection{ICQL-QS: A Stable but Myopic Baseline}
Our first algorithm, Individual Conservative Q-Learning with Q-sharing (ICQL-QS), is the most direct implementation of our core idea. While standard CQL penalizes OOD actions, we recognize that in MARL, a fully new joint action is almost always OOD. ICQL-QS therefore takes a more targeted approach by constructing a Bellman target that only changes one agent's action at a time.

This is formalized through an individual Bellman operator, $\mathcal{T}_i^{\text{ind}}$, for each agent $i$. This operator computes the expected future value by sampling agent $i$'s next action, $a'_i$, from its learned policy $\pi_i$, while sampling the actions of all other agents, $\boldsymbol{a}'_{-i}$, from the dataset $\mathcal{D}$:
\begin{equation}
    \mathcal{T}_i^{\text{ind}} Q(s, \boldsymbol{a}) := \mathbb{E}_{\substack{(s', \boldsymbol{a}'_{-i}) \sim \mathcal{D} \\ a'_{i} \sim \pi_{i}(\cdot|s')}} \left[r + \gamma Q(s', a'_{i}, \boldsymbol{a}'_{-i})\right].
\end{equation}
The full algorithm then learns a shared Q-function, $Q_\theta$, by minimizing the average of per-agent loss $\mathcal{L}_i(\theta)$:
\begin{equation}
\label{icql_qs_loss}
\begin{aligned}
    &\mathcal{L}_i(\theta) = \mathbb{E}_{(s, \boldsymbol{a}) \sim \mathcal{D}} \left[ (Q_\theta(s, \boldsymbol{a}) - \mathcal{T}_i^{\text{ind}} Q)^2 \right] \\
    & + \lambda \left( \mathbb{E}_{(s, \boldsymbol{a}_{-i})  \sim \mathcal{D}, a_i \sim \pi_i}[Q_\theta(s, a_i, \boldsymbol{a}_{-i})] - \mathbb{E}_{(s,\boldsymbol{a}) \sim \mathcal{D}}[Q_\theta(s, \boldsymbol{a})] \right).
\end{aligned}
\end{equation}
The first term is the squared Bellman error with respect to our individual operator, while the second is the conservative regularizer that penalizes OOD actions for agent $i$. 

Although each update appears to be ``myopic" by only considering a single agent's deviation, the use of a single, shared Q-function means that every update provides a learning signal that implicitly couples the agents and fosters coordination. We formalize this connection to a centralized learning objective in our theoretical analysis.

\begin{table}[!t]
\centering
\caption{\footnotesize Performance comparison between CFCQL and ICQL-QS}
\label{tab:mape_results}
\scriptsize
\begin{tabular}{@{}l*{8}{c}@{}}
\toprule
\multirow{2}{*}{Algorithm} & \multicolumn{4}{c}{Cooperative Navigation} & \multicolumn{4}{c}{World } \\
\cmidrule(lr){2-5} \cmidrule(lr){6-9} 
 & Exp & Med & Med-R & Rand & Exp & Med & Med-R & Rand \\
\midrule
CFCQL & \textbf{112} & \textbf{65} & 52.2 & 62.2 & \textbf{119.7} & \textbf{93.8} & 73.4 & 68 \\
ICQL-QS & 97.2 & 62.5 & \textbf{58.2} & \textbf{77.7} & 106.5 & 88.6 & \textbf{82.5} & \textbf{89.9} \\
\bottomrule
\end{tabular}
\label{compare_ICQLQS}
\end{table}

\begin{figure}[t]
  \centering
\includegraphics[width=0.95\columnwidth]{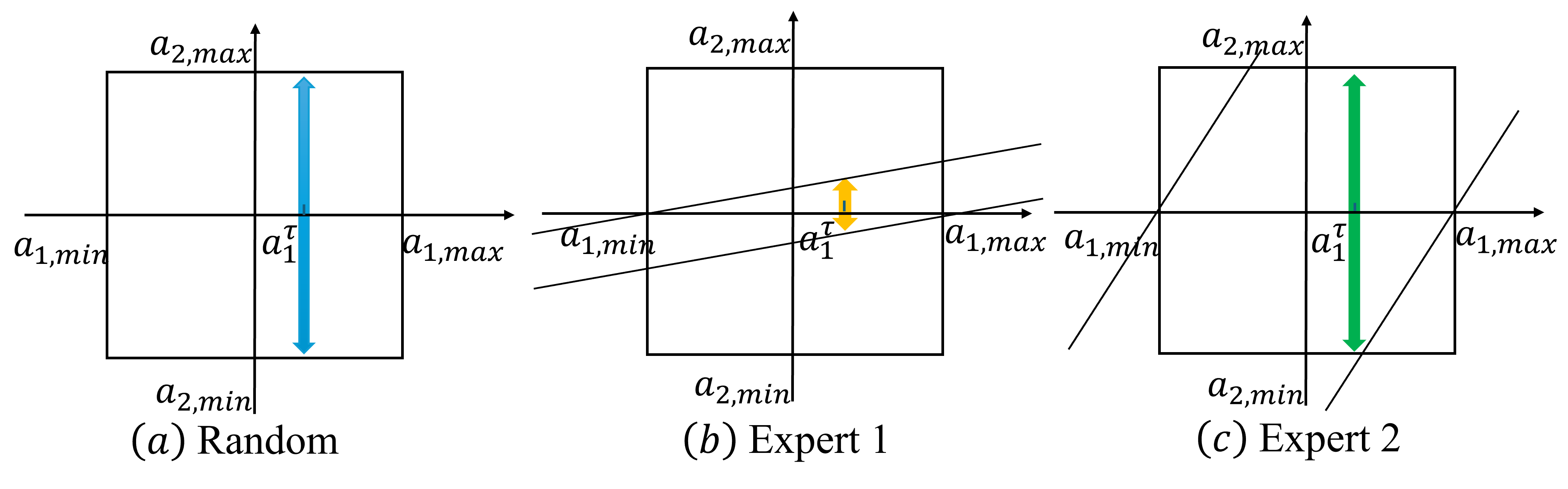}
% \caption{\footnotesize Illustration of potential action distribution range in datasets. $a_i^\tau$ is agent 1's action sampling from dataset. (a) In Random datasets, action distributions are usually independent. Agent 2's behavioral action might be any value within the value range. Therefore fixing agent 1's action can enjoy the wide distribution of agent 2's action. (b) and (c) show two cases of Expert dataset when agents' actions are linearly correlated. Agent 2's behavioral action might be within a narrow range like (b). In this case, fixing agent 1's action may still lead to severe OOD issue like joint action replacement. However, in the case like (c), fixing agent 1's action can still potentially enjoy the wide distribution of agent 2's action.  }
\caption{\footnotesize Action distributions in different dataset types. Blue, yellow and green lines show the range of agent actions, with $a_1^\tau$ denoting agent 1's action from the dataset. (a) \textbf{Random datasets}: Agents act independently, so agent 2's actions may span the full range regardless of agent 1's choice. Partial action replacement (fixing agent 1) allows reasonable coverage of agent 2's action space. (b) \textbf{Expert dataset with tight correlation}: Agents' actions are strongly correlated within a narrow band. Fixing agent 1 may still lead to severe OOD issue. (c) \textbf{Expert dataset with loose correlation}: While correlated, actions have sufficient spread. Fixing agent 1 still allows reasonable coverage of agent 2's action space. This illustrates the potential advantage of partial action replacement for random datasets and some expert datasets, but it may struggle with highly coordinated behaviors.}
\label{illustration}
\end{figure}

\subsection{SPaCQL: An Adaptive Mixture of Partial Backups}
While ICQL-QS achieves stability through single-agent updates, it cannot capture the value of coordinated multi-agent improvements. Yet fully joint updates risk severe extrapolation errors. The key insight motivating SPaCQL is that the optimal number of deviating agents depends on the data: random datasets favor conservative single-agent updates, while expert datasets may not, but may support coordinated deviations. Rather than committing to either extreme, SPaCQL adaptively interpolates between them. Our initial experiments confirmed this intuition, revealing a stark performance trade-off dependent on dataset quality.

The performance comparison between CFCQL \cite{Shao2023CounterfactualLearning} and ICQL-QS showed in Table \ref{compare_ICQLQS} indicates that on datasets composed of noisy, ``random," or exploratory behavior, ICQL-QS demonstrates superior performance. By staying close to the data manifold, they minimize distribution shift and avoid the OOD extrapolation errors that plague joint-action methods in sparse data regimes. Figure \ref{illustration} illustrates the potential advantage of partial action replacement in terms of mitigating OOD issue especially for random datasets. 

Conversely, on ``expert" datasets containing highly coordinated trajectories, CFCQL using a full joint-action update can outperform ICQL-QS. On the one hand, the advantage of ICQL-QS in mitigating OOD issue may not be prominent on expert dataset, as illustrated in Figure \ref{illustration} (b). On the other hand, evaluating fully novel joint actions is essential for learning the value of synergistic maneuvers, and a factorized approach is ``coordination-blind" in this setting and can become overly pessimistic. However, as illustrated in Figure \ref{illustration} (c), on some expert datasets, partial action replacement may still have a considerable advantage in mitigating OOD issue. 

Our experiments shown in Figure \ref{fig:uncertainty_comparison} confirm this advantage of mitigating OOD issue by measuring the uncertainty of Q-estimates. We empirically compare the estimation uncertainty of different target-value strategies, measured by the standard deviation across a Q-function ensemble. As shown, a standard joint-action update (CFCQL) exhibits higher uncertainty than ICQL-QS that uses partial action replacement.

This presents a clear challenge: no single, fixed backup strategy is optimal across all data regimes. This motivates our primary objective: to design a unified algorithm that can gracefully and automatically interpolate between these two extremes, letting the data itself decide how many coordinated deviations are safe to consider at any given moment.

%%%%%%%%%

\begin{figure}[t]
\centering
\footnotesize
\includegraphics[width=0.75\columnwidth]{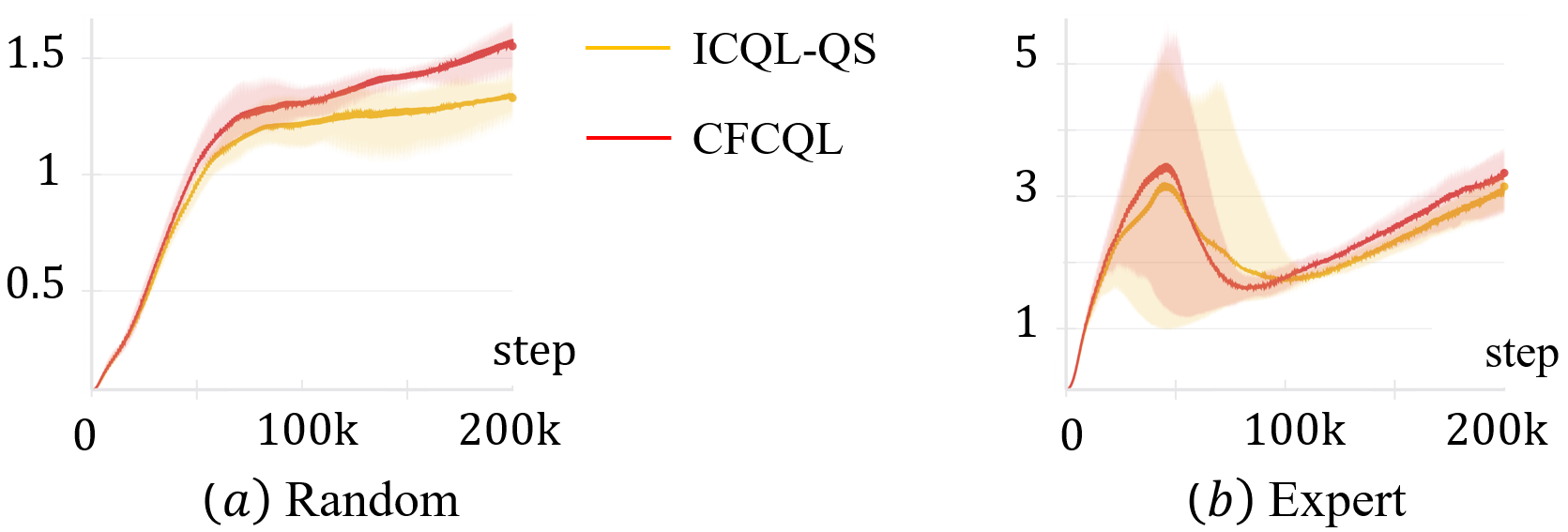}
% \begin{subfigure}[t]{0.35\linewidth}
%     \centering
%     \includegraphics[width=\linewidth]{fig/cn_random.png}
%     \caption{Random}
% \end{subfigure}
% % \hfill
% \begin{subfigure}[t]{0.35\linewidth}
%     \centering
%     \includegraphics[width=\linewidth]{fig/cn_expert.png}
%     \caption{Expert}
% \end{subfigure}
% \vspace{1ex} % Adds a small vertical space between rows
\caption{\footnotesize
    Comparison of the uncertainty of Q-value estimation between CFCQL and ICQL-QS on the Cooperative Navigation benchmark. Uncertainty is measured as the standard deviation of Q-estimates from an ensemble of networks. ICQL-QS, which uses partial action replacement yields significantly lower estimation uncertainty on Random dataset consistently. On Expert dataset, the two methods show similar estimation uncertainty, but finally ICQL-QS still leads to lower uncertainty.
}
\label{fig:uncertainty_comparison}
\end{figure}

%%%

\subsection{The Soft-Partial Bellman Operator}
Instead of committing to a fixed strategy, we introduce SPaCQL, which constructs its Bellman target as a mixture of $n$ base operators, $\{\mathcal{T}^{(k)}\}_{k=1}^n$, where each operator corresponds to replacing exactly $k$ agents' actions.

For any integer $k \in \{1, \dots, n\}$, we define a Bellman operator $\mathcal{T}^{(k)}$ that bootstraps from a next-state action where exactly $k$ agents (chosen uniformly at random) deviate from the behavior policy, while the other $n-k$ agents' actions are taken from the logged data.
\begin{equation}
    \mathcal{T}^{(k)}Q(s,\boldsymbol{a}) := \mathbb{E}_{s'\sim\mathcal{D}, \boldsymbol{a}'^{(k)}}\left[ r+\gamma\,Q\left(s',\,\boldsymbol{a}'^{(k)}\right)\right].
\end{equation}
Each of these operators is a $\gamma$-contraction in the $\ell_\infty$ norm.

The final SPaCQL operator, $\mathcal{T}^{\text{SP}}$, is a convex combination of these base operators, with mixture weights $w_k$:
% for $k \in \{1, \dots, n\}$:
\begin{equation} \label{eq:sp_operator_final}
\mathcal{T}^{\text{SP}}Q(s,\boldsymbol{a}) := \sum_{k=1}^{n} w_k\;\mathcal{T}^{(k)}Q(s,\boldsymbol{a}).
\end{equation}

% The final SPaCQL operator, $\mathcal{T}^{\text{SP}}$, is a convex combination of these base operators, where the mixture weights $w = (w_0, \dots, w_n)$:
% \begin{equation} \label{eq:sp_operator_final}
%     \mathcal{T}^{\text{SP}}Q(s,\boldsymbol{a}) := \sum_{k=1}^{n} w_k\;\mathcal{T}^{(k)}Q(s,\boldsymbol{a}).
% \end{equation}

As a convex combination of $\gamma$-contractions, $\mathcal{T}^{\text{SP}}$ is itself guaranteed to be a $\gamma$-contraction.

\subsection{The SPaCQL Algorithm and Learning Objective}
The learning objective for SPaCQL, $\mathcal{L}(\theta)$, which combines the TD error for our adaptive Bellman operator, is: 
\begin{equation}\label{eq:spacql_loss_final}
    \mathcal{L}(\theta) = \mathbb{E}_{\mathcal{D}} \left[ \left(Q_\theta(s, \boldsymbol{a})-Y_{\text{SP}}\right)^2 \right] + \xi_c,
\end{equation}
where $\xi_c= \alpha \sum_{i=1}^n \lambda_i ( \mathbb{E}_{(s, \boldsymbol{a}_{-i}) \sim \mathcal{D}, a_i \sim \pi_i}[Q_\theta(s, a_i, \boldsymbol{a}_{-i})] - \mathbb{E}_{(s,\boldsymbol{a}) \sim \mathcal{D}}[Q_\theta(s, \boldsymbol{a})] )$ is a conservative penalty same as that used in CFCQL.

The target $Y_{\text{SP}}$ is constructed using the soft-partial operator, with an ensemble for conservatism:
\begin{equation}
    Y_{\text{SP}} := r+\gamma\,\sum_{k=1}^{n} w_k\; \min_{j}Q^{\text{tar}}_j\bigl(s',\boldsymbol{a}'^{(k)}\bigr).
\end{equation}

The weights are determined by uncertainty: high ensemble disagreement signals poor data coverage, so we downweight the corresponding risky deviations. We measure uncertainty via Q-ensemble variance:
\begin{equation}
\label{std_Q}
u_k = \sqrt{\operatorname{Var}_{j}\bigl[Q_{\theta_j}(s', \boldsymbol{a}'^{(k)})\bigr]}.
\end{equation}
We then map this uncertainty signal to the weight vector using the normalized standard deviation of the Q-values at the next state: $w_k = \frac{1/u_k}{\sum_k 1/u_k}$, such that higher variance pushes weight to more conservative, smaller-$k$ backups. 
A complete algorithm is shown in Algorithm \ref{alg:spacql}.

%%%%

\begin{algorithm}[t!]
\footnotesize
    \caption{SPaCQL}
    \label{alg:spacql}
    \begin{algorithmic}[1]
    \STATE \textbf{Initialize:} Q‑ensemble $\{Q_{\theta_j}\}$, target ensemble $\{\bar Q_{\bar\theta_j}\}$, policies $\{\pi_{\psi_i}\}_{i=1}^n$, target policies $\{\bar\pi_{\bar\psi_i}\}_{i=1}^n$, replay buffer $\mathcal{D}$
    \FOR{each iteration}
        \STATE Sample batch $\mathcal{B}=\{(s,\boldsymbol{a},r,s',\boldsymbol{a}')\}$ from $\mathcal{D}$
        \STATE Set $\mathcal{L}(\theta)\leftarrow0$
        \FOR{each transition $(s,\boldsymbol{a},r,s',\boldsymbol{a}')$ in $\mathcal{B}$}
            % \STATE $\boldsymbol{a}'^{(0)} \leftarrow \boldsymbol{a}'$ %\COMMENT{The base action with k=0 replacements}
            \FOR{$k=1$ to $n$}
                \STATE Sample $k$ agent indices $\{\sigma_\rho\}_{\rho=1}^k$
                \STATE Sample $\{a_{\sigma_\rho}^{\pi} \sim \pi_{\sigma_\rho}(\cdot|s')\}_{\rho=1}^k$
                \STATE Construct $\boldsymbol{a}'^{(k)} \leftarrow \boldsymbol{a}' $:   $\forall \sigma_\rho \in \{\sigma_\rho\}_{\rho=1}^k$, replace the $\sigma_\rho$-th component of $\boldsymbol{a}'^{(k)}$ with $a_{\sigma_\rho}^{\pi}$
                %\STATE Calculate $u_k$ using (\ref{std_Q}) 
                \STATE $u_k \leftarrow \sqrt{\operatorname{Var}_{j}\bigl[Q_{\theta_j}(s', \boldsymbol{a}'^{(k)})\bigr]}$
                
                \STATE Set $y_k= \frac{1}{u_k} \min_j\bar Q_{\theta_j}(s',\boldsymbol{a}'^{(k)})$
            \ENDFOR
            \STATE $Y_{\text{SP}} = r + \gamma \sum_k y_k / \sum_k \frac{1}{u_k}$
            \STATE $\mathcal{L}(\theta) \mathrel{+}= \sum_{j} \bigl(Q_{\theta_j}(s,\boldsymbol{a})-Y_{\text{SP}}\bigr)^2 + \xi_c$
        \ENDFOR
        % \STATE $L_{\text{TD}}\mathrel{+}= \xi_c$
        \STATE $\theta\leftarrow\theta-\eta_\theta\nabla_\theta \mathcal{L}(\theta)$ 
        \STATE $\bar\theta\leftarrow(1-\tau)\bar\theta+\tau \theta$
        \STATE \textcolor{gray}{\textit{--- Agent policy update ---}} 
        \FOR{each agent $i$}
            \STATE $\psi_i \leftarrow \psi_i + \eta_\pi\nabla_{\pi_i} \mathbb{E}_{s, \boldsymbol{a}_{-i}\sim \mathcal{D}, a_i\sim \pi_i} Q_{\theta_1}(s, \boldsymbol{a})$
            \STATE $\bar\psi_i\leftarrow(1-\tau)\bar\psi_i+\tau \psi_i$
        \ENDFOR   
    \ENDFOR
    \end{algorithmic}
\end{algorithm}

%%%%%%%%%%%%%%%%%%%%%%%%%
\section{Theoretical Analysis}\label{sec:theory}

We provide a theoretical foundation explaining why \emph{partial action replacement} alleviates distribution shift in offline MARL. Our key insight is that while the joint-action space grows exponentially in the number of agents, the distribution shift induced by partial replacement grows only linearly. All proofs are in Appendix \ref{app:proofs}.

\subsection{Setting and key assumptions}

We consider a finite Dec-MDP $\langle\mathcal{S},\{\mathcal{A}_i\}_{i=1}^n,P,R,\gamma\rangle$ and a fixed dataset $\mathcal{D}=\{(s,\boldsymbol{a},r,s')\}$ of \emph{i.i.d.} transitions. Rewards are bounded: $|R(s,\boldsymbol{a})|\le 1$.

\begin{assumption}[Factorized behavior policy]
\label{ass:factorized}
The behavior policy is factorized: $\mu(s,\boldsymbol{a})=\prod_{i=1}^{n}\mu_i(a_i\mid s)$.
\end{assumption}

The finite state-action space assumption is standard in theoretical MARL analysis and allows us to use discrete probability measures and avoid measure-theoretic complications. The i.i.d. assumption on dataset transitions simplifies the analysis by avoiding temporal correlations, though our practical algorithm works with standard trajectory data. Assumption~\ref{ass:factorized} is crucial—it means the data collection process involved independent agent policies, which is realistic for many multi-agent scenarios where agents were trained separately or act without explicit coordination. This assumption enables us to decompose distribution shifts across individual agents, which is fundamental to our analysis. The bounded rewards assumption is standard and can always be achieved through reward clipping; it ensures value functions remain bounded and Lipschitz continuous.

\paragraph{Metric and total variation.} On the finite state-action space we use the $0$-$1$ metric. We define total variation distance as $\mathrm{TV}(\nu,\nu') := \frac{1}{2}\|\nu-\nu'\|_1$. Under the $0$-$1$ metric, the \emph{Wasserstein-1 distance equals total variation}:
\begin{equation}
W_1(\nu,\nu') = \mathrm{TV}(\nu,\nu') = \frac{1}{2}\|\nu-\nu'\|_1
\end{equation}

The choice of $0$-$1$ metric simplifies the analysis significantly—in discrete spaces, it makes Wasserstein distance equivalent to total variation distance, eliminating the need for optimal transport calculations while still capturing meaningful notions of distribution similarity.

\paragraph{Lipschitz constant of $Q$.} With bounded rewards, the optimal value function obeys $\|Q^\star\|_\infty\le 1/(1-\gamma)$ and is $2/(1-\gamma)$-Lipschitz under the $0$-$1$ metric. To see this, note that for any two state-action pairs $(s,\boldsymbol{a})$ and $(s',\boldsymbol{a}')$:
\begin{equation}
|Q^\star(s,\boldsymbol{a}) - Q^\star(s',\boldsymbol{a}')| \le \sum_{t=0}^{\infty} \gamma^t \cdot 2 = \frac{2}{1-\gamma}
\end{equation}
since rewards are bounded by 1 and the $0$-$1$ metric ensures the maximum difference is 2.

This Lipschitz property is automatic given bounded rewards and provides the key tool for translating distribution shifts into value function estimation errors. The specific constant $2/(1-\gamma)$ is tight and captures how value function smoothness degrades as the discount factor approaches 1.

\paragraph{Mixed policies.} For any subset $S\subseteq\{1,\dots,n\}$, define:
\begin{equation}
\footnotesize
\pi^{(S)} := \left(\prod_{i\in S}\pi_i\right)\left(\prod_{j\notin S}\mu_j\right)
\end{equation}
and let $d^{(S)}$ denote its discounted occupancy measure.

These mixed policies are the key technical tool for our analysis. $\pi^{(S)}$ represents a policy where agents in set $S$ follow their learned policies $\pi_i$ while agents outside $S$ follow the behavior policies $\mu_j$. This construction allows us to decompose the transition from behavior policy $\mu = \pi^{(\varnothing)}$ to learned policy $\pi = \pi^{(\{1,\ldots,n\})}$ through a sequence of intermediate policies, enabling us to track how distribution shift accumulates as we add more deviating agents.

%%%%%%

\subsection{Distribution-Shift Control via Partial Replacement}

\paragraph{Total-variation convention.}
For two stochastic kernels $\kappa,\kappa'$ on $\mathcal{S}$ we write
\[
   \mathrm{TV}(\kappa,\kappa') :=
      \sup_{s\in\mathcal{S}}\,
      \frac{1}{2}\|\kappa(s,\cdot)-\kappa'(s,\cdot)\|_1 .
\]
Under the $0$-$1$ ground metric, $W_1(\kappa,\kappa')=\mathrm{TV}(\kappa,\kappa')$.

\begin{lemma}[Linear Divergence Bound]
\label{lem:linear-div}
For every $S\subseteq\{1,\dots,n\}$:
\begin{equation}
W_1\left(d^{(S)},d^{(\varnothing)}\right) \le \frac{\gamma}{1-\gamma}\sum_{i\in S}\mathrm{TV}(\pi_i,\mu_i)
\end{equation}
\end{lemma}

The proof shows that despite the exponential joint-action space, the occupancy shift grows \emph{additively} with the number of deviating agents.

This result formalizes our core intuition: while there are $|\mathcal{A}|^n$ possible joint actions, the distribution shift from changing $k$ agents' policies is only $k$ times larger than changing a single agent's policy. This linear scaling is the theoretical foundation for why partial action replacement is fundamentally more stable than full joint-action updates.

%%%%%%%%%%%%

\subsection{Value‑Estimation Error}

% \begin{theorem}[Tight value-error bound]\label{thm:value}
% Let $\hat{Q}$ be the Q-function learned offline and $\hat{V}^\pi=\mathbb{E}_{d^\pi,\pi}\hat{Q}$. Assume $\hat{Q}$ is $\frac{2}{1-\gamma}$-Lipschitz under the $0$-$1$ metric (this holds when using target networks or any contractive update rule that preserves the Lipschitz constant of $Q^\star$). Then:

\begin{theorem}[Tight Value-Error Bound]\label{thm:value}
Let $\hat{Q}$ be the Q-function learned offline and $\hat{V}^\pi=\mathbb{E}_{d^\pi,\pi}\hat{Q}$. We assume that the learning process ensures $\hat{Q}$ is $\frac{2}{1-\gamma}$-Lipschitz under the $0$-$1$ metric. While this property is not automatically guaranteed for neural network approximators, it can be encouraged with techniques like spectral normalization or gradient clipping and serves as a key theoretical assumption. Then:

\begin{align}
|V^\pi-\hat{V}^\pi| &\le \varepsilon_{\mathrm{Subopt}} + \varepsilon_{\mathrm{FQI}} + \frac{4\gamma}{(1-\gamma)^2}\sum_{i=1}^{n}\mathrm{TV}(\pi_i,\mu_i)
\end{align}
where $\varepsilon_{\mathrm{Subopt}}=\|Q^\pi-Q^\star\|_\infty$ and $\varepsilon_{\mathrm{FQI}}=\|Q^\star-\hat{Q}\|_\infty$.
\end{theorem}
The proof is provided in Appendix \ref{app:proofs}.

\paragraph{Remark.}
If only a single agent $k$ deviates from $\mu$, the last term reduces to
$\,\tfrac{4\gamma}{(1-\gamma)^2}\mathrm{TV}(\pi_k,\mu_k)$,
strictly improving on the usual joint‑TV bound.

%%%%%%%%

\subsection{Generalization to Correlated Behavior Policies}

The guarantee in Theorem~\ref{thm:value} relies on the assumption that the behavior policy $\mu$ is fully factorized. We now relax this assumption and show that the linear scaling of distribution shift is robust, degrading gracefully with the degree of correlation in the dataset.

First, we define a measure of the maximal correlation present in the behavior policy. Let $\mu^\otimes(\boldsymbol{a}\mid s) := \prod_{i=1}^n \mu_i(a_i\mid s)$ be the product of the marginal behavior policies.

\begin{definition}[Maximal Excess Correlation]
The maximal excess correlation, $\kappa$, is the largest Total Variation distance between the true joint behavior policy and its factorized approximation, over all states:
$$\kappa := \sup_{s\in\mathcal S}\; \frac{1}{2}\bigl\lVert \mu(\,\cdot\mid s) - \mu^\otimes(\,\cdot\mid s) \bigr\rVert_1$$
where $\kappa \in [0, 1]$. If $\mu$ is perfectly factorized, then $\kappa=0$.
\end{definition}

With this, we can state a more general version of our distribution-shift bound.

\begin{lemma}[Linear Divergence with Correlations]\label{lem:linear-div-corr}
For any subset of agents $S\subseteq\{1,\dots,n\}$, the distribution shift is bounded by:
\begin{equation}
   W_{1}\bigl(d^{(S)},d^\mu\bigr) \;\le\; \frac{\gamma}{1-\gamma}\left( \sum_{i\in S} \operatorname{TV}(\pi_i,\mu_i) + \kappa \right)
\end{equation}
\end{lemma}

This generalized lemma allows us to bound the value-estimation error without the factorization assumption.

\begin{theorem}[Value-Error Bound with Correlations]\label{thm:value-corr}
Under the same conditions as Theorem~\ref{thm:value}, but without the assumption of a factorized behavior policy, the value estimation error is bounded by:
\begin{align} \label{eq:value-corr-bound-alt}
    |V^\pi-\hat{V}^\pi| \le{}& \varepsilon_{\mathrm{Subopt}} + \varepsilon_{\mathrm{FQI}} \nonumber \\
    & + \frac{4\gamma}{(1-\gamma)^2}\left(\sum_{i=1}^{n}\mathrm{TV}(\pi_i,\mu_i) + \kappa\right)
\end{align}
\end{theorem}
\begin{proof}
The proof is identical to that of Theorem~\ref{thm:value}, substituting the bound from Lemma~\ref{lem:linear-div-corr}.
\end{proof}

This result shows that correlations in the behavior policy introduce a simple additive penalty $\kappa$ to the error bound. Crucially, this penalty is independent of the number of agents, $n$. This means the ``curse of dimensionality" is still avoided; the bound remains linear in $n$ and is simply shifted upwards by a constant reflecting the data's inherent correlation. The theory therefore remains informative even when the factorization assumption is violated.

%%%%%%%%%%
\subsection{Theoretical Guarantee of SPaCQL}
SPaCQL retains the strong theoretical guarantees of our framework without requiring any new structural assumptions on the Q-function. The expected distribution shift under the SPaCQL operator is linear in the \textit{effective} number of deviating agents, $k_{\text{eff}} = \sum_k w_k \cdot k$. This leads directly to our final value-error bound.

\begin{theorem}[SPaCQL Value-Error Bound]
Let $\hat{Q}$ be the function learned by SPaCQL and define the average single-agent policy deviation as $\overline{\mathrm{TV}}(\pi, \mu) = \frac{1}{n} \sum_{i=1}^n \mathrm{TV}(\pi_i, \mu_i)$. The value estimation error is bounded by:
\footnotesize
$$
|V^\pi-\hat{V}^\pi| \le \varepsilon_{\mathrm{Subopt}} + \varepsilon_{\mathrm{FQI}} + \frac{4\gamma}{(1-\gamma)^2}\, \mathbb{E}_{s' \sim d^\pi} \left[k_{\mathrm{eff}}(s')\right]\, \overline{\mathrm{TV}}(\pi, \mu).
$$
\end{theorem}

% \begin{theorem}[SPaCQL Value-Error Bound]
% Let $\hat{Q}$ be the function learned by SPaCQL. The value estimation error is bounded by:
% $$
% |V^\pi-\hat{V}^\pi| \le \varepsilon_{\mathrm{Subopt}} + \varepsilon_{\mathrm{FQE}} + \frac{4\gamma}{(1-\gamma)^2}\, \mathbb{E}_{s' \sim d^\pi} \left[k_{\mathrm{eff}}(s')\right]\,\overline{\mathrm{TV}}.
% $$
% \end{theorem}

% \begin{proof}
% The proof is identical to that of Theorem~\ref{thm:value}, substituting $\pi$ with $\sum_k {w_k} \pi^{(k)}$. {\color{red}{do we need to expand this?}}
% \end{proof}

\begin{proof}[Proof Sketch]
The proof follows the same structure as that of Theorem~\ref{thm:value}. The key difference is that the distribution shift term $W_1(d^\pi, d^\mu)$ is replaced by the expected shift under the SPaCQL operator. This expectation is taken over the mixture of policies $\sum_k w_k(s') \pi^{(k)}$, leading to a dependency on the state-dependent effective number of deviations, $k_{\mathrm{eff}}(s')$.
\end{proof}

The error scales with the effective number of deviations, $k_{\text{eff}}$, which the algorithm adapts on a state-by-state basis. When the weights concentrate on $k=1$, we recover the tight ICQL-QS bound; when they shift towards $k=n$, we approach the looser, full-joint bound, thus formalizing the algorithm's adaptive nature.

%%%%%%

\subsection{Gradient Interpretation of ICQL-QS}

While Lemma~\ref{lem:linear-div} and Theorem~\ref{thm:value} establish the stability benefits of partial action replacement, a natural question arises: can partial replacement achieve effective coordination? One might worry that updating Q-function separately with loss (\ref{icql_qs_loss}), even with a shared Q-function, may lead to miscoordinated policies that fail to capture beneficial joint behaviors.

The following result addresses this concern by revealing the implicit coordination mechanism in ICQL-QS:

\begin{proposition}[Gradient equivalence of ICQL-QS]\label{prop:grad-equiv}
Let $\mathcal{T}_i^{\mathrm{ind}}$ be the individual Bellman operator and define the \emph{averaged-individual} operator:
\begin{equation}
\mathcal{T}^{\mathrm{ai}} := \frac{1}{n}\sum_{i=1}^{n}\mathcal{T}_i^{\mathrm{ind}}
\end{equation}
Then the parameter update for the TD component in ICQL-QS is equivalent to stochastic gradient descent on centralized TD loss, under the semi-gradient assumption where the gradient is not backpropagated through the target network (i.e., treating $\mathcal{T}^{\mathrm{ai}}Q_\theta$ as a fixed target with respect to $\theta$):
% Then the parameter update for the TD component in ICQL-QS (using semi-gradient descent where the target $\mathcal{T}^{\mathrm{ai}}Q_\theta$ is treated as constant w.r.t. $\theta$) is equivalent to stochastic gradient descent on the single centralized TD loss:
\begin{equation}
\frac{1}{2}\mathbb{E}_{\mathcal{D}}\left[\left(Q_\theta-\mathcal{T}^{\mathrm{ai}}Q_\theta\right)^2\right]
\end{equation}
\end{proposition}

This result shows that ICQL-QS, despite appearing to perform individual agent updates, is mathematically equivalent to centralized training on an averaged objective. The shared Q-function acts as an implicit coordination mechanism: each agent's update influences the value estimates for all possible joint actions, enabling coordination without explicit joint reasoning. This provides theoretical justification for why partial action replacement can maintain coordination benefits while achieving the stability advantages established in our previous results.

%%%%%%%%%%%%%%%%%%%%%%%%%%%%%%%%%%%%%
\begin{table*}[!t]
\centering
\caption{The average normalized score on offline MARL tasks. The best performance is highlighted in bold. }
\label{tab:main_results}
\footnotesize
\begin{tabular}{c|ccccccccc}
\toprule
 & & Dataset & OMAR & MACQL & IQL & MA-TD3+BC & DoF & CFCQL & SPaCQL \\ \hline
\multirow{12}{*}{\rotatebox[origin=c]{90}{MPE}}   & \multirow{4}{*}{CN} & Exp  & 114.9 $\pm$ 2.6   & 12.2$\pm$31  & 103.7$\pm$2.5 & 108.3$\pm$3.3 &\textbf{136.4$\pm$3.9} &  112$\pm$4 &111.9$\pm$4.5   \\
                          &                                         & Med  & 47.9$\pm$18.9  & 14.3$\pm$20.2  & 28.2$\pm$3.9  & 29.3$\pm$4.8 & 75.6$\pm$8.7  & 65.0$\pm$10.2 &{ {\textbf{78.6$\pm$6.4}}} \\
                          &                                           & Med-R   & 37.9$\pm$12.3  & 25.5$\pm$5.9  & 10.8$\pm$4.5 & 15.4$\pm$5.6 & 57.4$\pm$6.8& 52.2$\pm$9.6  &{ {\textbf{71.9$\pm$13.2}}} \\
                          &                                           & Rand  & 34.4$\pm$5.3  & 45.6$\pm$8.7  & 5.5$\pm$1.1 & 9.8$\pm$4.9 & 35.9$\pm$6.8 & 62.2$\pm$8.1 &{ {\textbf{78.2$\pm$14}}} \\ \cline{2-10} 
                          
                          & \multirow{4}{*}{PP}            & Exp  & 116.2$\pm$19.8  & 108.4$\pm$21.5  & 109.3$\pm$10.1 & 115.2$\pm$12.5  &\textbf{125.6$\pm$8.6} & 118.2$\pm$13.1 &111.2$\pm$16.4 \\
                          &                                           & Med  & 66.7$\pm$23.2  & 55$\pm$43.2  & 53.6$\pm$19.9 & 65.1$\pm$29.5  & \textbf{86.3$\pm$10.6} & 68.5$\pm$21.8  &61.9$\pm$20 \\
                          &                                           & Med-R   & 47.1$\pm$15.3  & 11.9$\pm$9.2  & 23.2$\pm$12 & 28.7$\pm$20.9  & 65.4$\pm$12.5 & 71.1$\pm$6  &{ {\textbf{75.0$\pm$12.7}}} \\
                          &                                           & Rand  & 11.1$\pm$2.8  & 25.2$\pm$11.5  & 1.3$\pm$1.6 & 5.7$\pm$3.5  &16.5$\pm$6.3 & 78.5$\pm$15.6  &{ {\textbf{89.4$\pm$13.7}}} \\ \cline{2-10} 
                          & \multirow{4}{*}{World}                    & Exp  & 110.4$\pm$25.7  & 99.7$\pm$31  & 107.8$\pm$17.7 & 110.3$\pm$21.3  & \textbf{135.2$\pm$19.1} & 119.7$\pm$26.4  &112.3$\pm$7.8 \\
                          &                                           & Med  & 74.6$\pm$11.5  & 67.4$\pm$48.4  & 70.5$\pm$15.3 & 73.4$\pm$9.3 & 85.2$\pm$11.2 & 93.8$\pm$31.8 &{ {\textbf{98.1$\pm$17.7}}} \\
                          &                                           & Med-R   & 42.9$\pm$19.5  & 13.2$\pm$16.2  & 41.5$\pm$9.5 & 17.4$\pm$8.1 & 58.6$\pm$10.4 & 73.4$\pm$23.2 &{ {\textbf{105.2$\pm$11.1}}} \\
                          &                                           & Rand  & 5.9$\pm$5.2  & 11.7$\pm$11  & 2.9$\pm$4.0 & 2.8$\pm$5.5 & 13.1$\pm$2.1 & 68$\pm$20.8  &{ {\textbf{94.3$\pm$7.4}}} \\ \hline
\multirow{4}{*}{\rotatebox[origin=c]{90}{MaMujoco}} & \multirow{4}{*}{Half-C}              & Exp  & 113.5$\pm$4.3  & 50.1$\pm$20.1  & 115.6$\pm$4.2  & 114.4$\pm$3.8 & - & \textbf{118.5$\pm$4.9}  & 110.5$\pm$5.9 \\
                          &                                           & Med  & 80.4$\pm$10.2  & 51.5$\pm$26.7  & \textbf{81.3$\pm$3.7}  & 75.5$\pm$3.7 & -& 80.5$\pm$9.6  &70.3$\pm$7.8 \\
                          &                                           & Med-R   & 57.7$\pm$5.1  & 37.0$\pm$7.1  & 58.8$\pm$6.8  & 27.1$\pm$5.5 & -& 59.5$\pm$8.2  &{ {\textbf{66.1$\pm$3.4 }}}\\
                          &                                           & Rand  & 13.5$\pm$7.0  & 5.3$\pm$0.5  & 7.4$\pm$0.0  & 7.4$\pm$0.0 & - & 39.7$\pm$4.0  &{ {\textbf{43.8$\pm$4.9}}} \\ 
                          \bottomrule
\end{tabular}
\end{table*}

\section{Experimental settings and results}

\paragraph{Experimental setup} We conduct a comprehensive evaluation of SPaCQL on two widely used offline MARL benchmarks: Multi-Agent Particle Environments (MPE) and Multi-Agent MuJoCo (MaMujoco). We use the same dataset as in recent works \cite{Shao2023CounterfactualLearning, Pan2022PlanRectification, Kostrikov2022OfflineQ-Learning}.   For MPE, there are three distinct tasks: Cooperative Navigation (CN), Predator-Prey (PP), and World. For MaMujoco, there is Half-Cheetah (Half-C). To assess performance under varying data quality, each task utilizes four dataset types: Expert (Exp), Medium (Med), Medium-Replay (Med-R), and Random (Rand).

Our implementation uses an ensemble of 10 Q-networks to estimate value functions and their uncertainty. To ensure robust results, we run all experiments with 5 different random seeds and report the mean and standard deviation of the normalized scores. All hyperparameters for our method and the CFCQL baseline are identical to those specified in the original CFCQL paper. All experiments are implemented with pytorch and
run on NVIDIA Tesla V100 GPUs.

\paragraph{Baselines}
The performance of SPaCQL is compared against a suite of state-of-the-art offline MARL algorithms. These baselines are chosen to represent the primary avenues of research in the field:
\begin{itemize}
    \item Policy-constrained methods: OMAR \cite{Pan2022PlanRectification}, IQL \cite{Kostrikov2022OfflineQ-Learning}, and MA-TD3+BC \cite{Fujimoto2021ALearning, Pan2022PlanRectification}.
    \item Value-constrained methods: MACQL \cite{Shao2023CounterfactualLearning} and CFCQL \cite{Shao2023CounterfactualLearning}.
    \item Diffusion model-based methods: DoF \cite{Li2025DOF:LEARNING}, a recent approach using diffusion models.
\end{itemize}

\paragraph{Main quantitative results}
The average normalized scores across all tasks and datasets are presented in Table \ref{tab:main_results}, where all baseline algorithm scores are reported as in the corresponding papers. Overall, SPaCQL outperforms all baseline algorithms on 10 of the 16 tasks. The most significant performance gains are observed on datasets with low-quality or uncoordinated data. SPaCQL consistently demonstrates remarkable superiority over all baselines on every ``Random" and ``Medium-Replay" dataset across all four tasks (CN, PP, World, and Half-C). For instance, on the World task with the ``Random" dataset, SPaCQL achieves a score of 94.3 $\pm$ 7.4, whereas the next-best baseline, CFCQL, scores only 68 $\pm$ 20.8.

On high-quality ``Expert" datasets, the performance is comparable across several algorithms. Here, methods like DoF, which is the top performer on all three MPE Expert datasets, show strong results. This highlights that when the dataset already contains highly coordinated trajectories, keeping close to the underlying policies is enough.

\paragraph{Analysis of adaptive weights}
To verify that SPaCQL's adaptive mechanism functions as intended, we visualize the learned weights in Figure~\ref{fig:weights_comparison}. The results show the algorithm dynamically adjusts its strategy based on dataset quality. On unstructured datasets like ``Random" and ``Medium-Replay", the weight for single-agent deviations ($w_1$) is dominant, prioritizing stability by staying close to the data manifold. Conversely, on ``Expert" datasets, the weights for coordinated deviations ($w_2$ and $w_3$) increase, showing that SPaCQL shifts its focus to find better coordinated policies when the data suggests coordination is reliable.

%%%%%%%%%%%%%%%%%%%%%%%%%%%%%%%%%%%%
\begin{figure}[h]
\centering
\footnotesize
\begin{subfigure}[t]{0.4\linewidth}
    \centering
    \includegraphics[width=\linewidth]{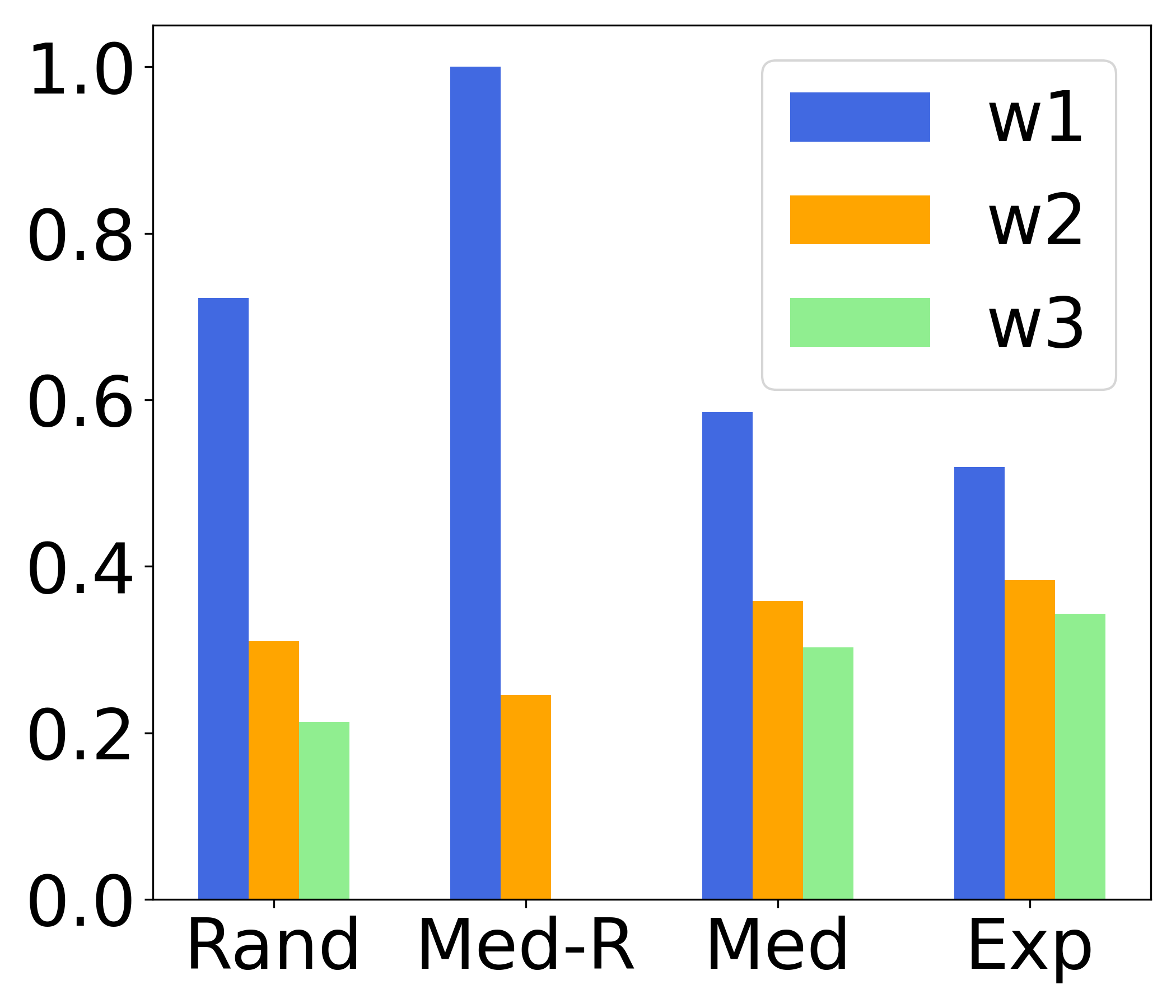}
    \caption{World}
\end{subfigure}
% \hfill
\hspace{3ex}
\begin{subfigure}[t]{0.4\linewidth}
    \centering
    \includegraphics[width=\linewidth]{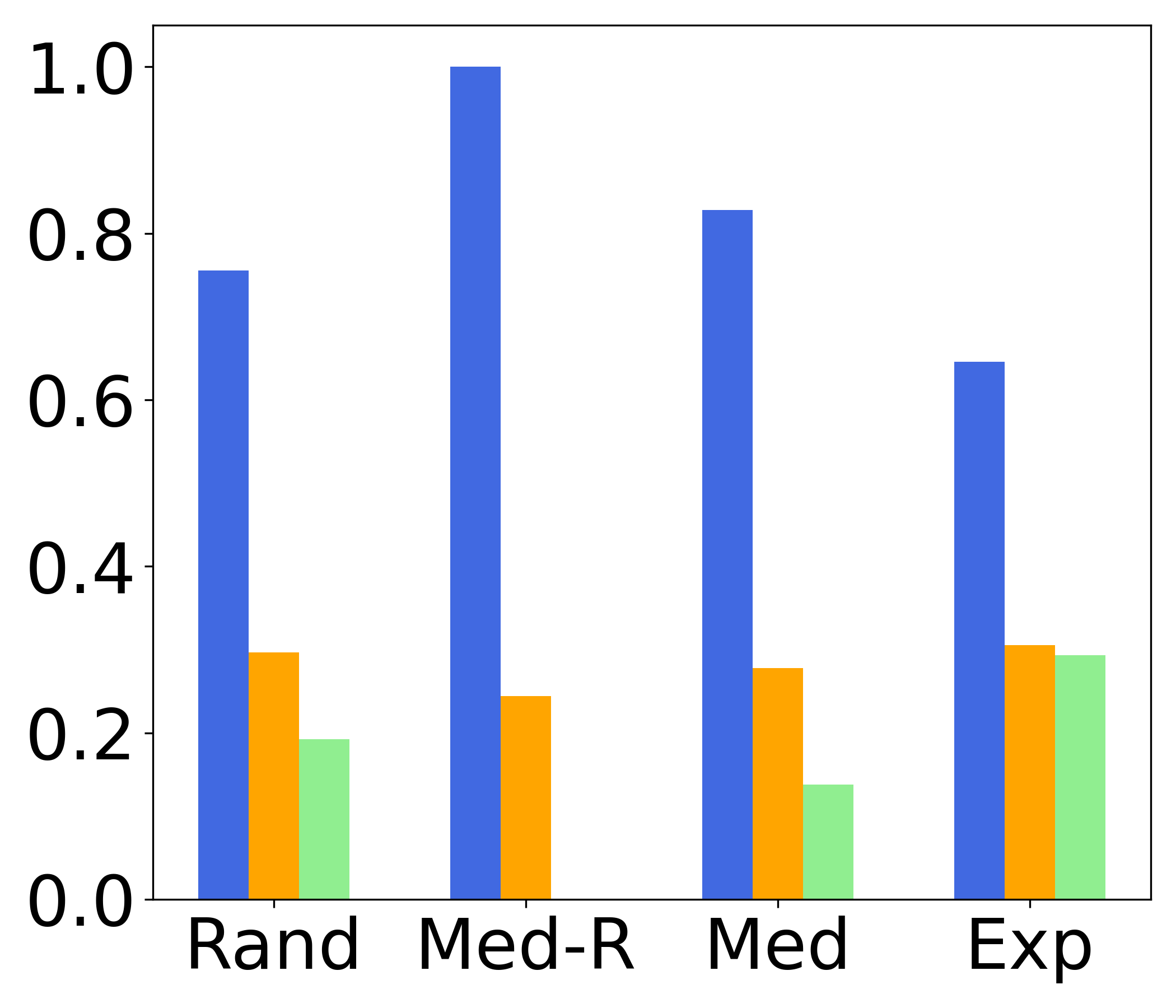}
    \caption{CN}
\end{subfigure}
% \begin{subfigure}[t]{0.23\linewidth}
%     \centering
%     \includegraphics[width=\linewidth]{fig/weights_cn.png}
%     \caption{Medium}
% \end{subfigure}
% % \hfill
% \begin{subfigure}[t]{0.23\linewidth}
%     \centering
%     \includegraphics[width=\linewidth]{fig/weights_cn.png}
%     \caption{Med-R}
% \end{subfigure}
% \vspace{1ex} % Adds a small vertical space between rows
\caption{\footnotesize
    Values of weights with min-max normalization.
}
\label{fig:weights_comparison}
\end{figure}

\section{Discussion and Conclusions}

Our theoretical analysis reveals a key insight: distribution shift in offline MARL scales linearly with the number of deviating agents, not exponentially with the joint-action space. This striking result suggests the ``curse of dimensionality'' in this setting may be overstated, as the distribution shift penalty accumulates additively across agents.

Our algorithm, SPaCQL, is built on this principle. It explicitly manages the stability-coordination trade-off by using partial action replacement, dynamically weighting different coordination combinations based on uncertainty to achieve a provably tighter value-error bound.

This work suggests a broader principle for offline MARL: algorithms should minimize simultaneous deviations from the behavior policy. While complete theoretical safety guarantees are pending analysis, the empirical motivation for SPaCQL is strong, as it adapts to the dataset by dynamically selecting the most reliable update strategy. Important future directions therefore include developing more sophisticated weighting schemes and uncertainty estimation techniques.

By formalizing partial action replacement and demonstrating its theoretical advantages, we provide a more optimistic perspective on the challenges of offline MARL. The linear scaling of distribution shift suggests that principled, localized approaches can overcome seemingly intractable combinatorial barriers in multi-agent learning.

\bibliography{main}

%%%%%%%%%%%%%%%%%%%%%%%%%%%%%%%%%%%%%

\newpage

\appendix

\onecolumn

%%%%%%%%%%%%%%%%%%%%%%%%%%%%%%%%%%%%%%%%%%%%%%%%%%%%%%%%%%%%%%%%%%%%%%%%

\section{Proofs}
\label{app:proofs}
\subsection{Notation recap}%
\label{app:notation}
For a policy $\phi$ let $d^\phi$ be its discounted occupancy measure and $P^\phi$ the induced transition kernel. For two vectors $v,w$, $\lVert v\rVert_1=\sum_x|v(x)|$. We use the total variation convention $\mathrm{TV}(v,w)=\frac{1}{2}\lVert v-w\rVert_1$.

For any subset $S\subseteq\{1,\dots,n\}$, recall that $\pi^{(S)} := \left(\prod_{i\in S}\pi_i\right)\left(\prod_{j\notin S}\mu_j\right)$ and $d^{(S)}$ denotes its discounted occupancy measure.

%-----------------------------------------------------------------------

\subsection{Proof of Lemma~\ref{lem:linear-div}}

\begin{proof}
We establish the notation $P^\pi(s'|s) = \sum_{\boldsymbol{a}} P(s'|s,\boldsymbol{a})\pi(\boldsymbol{a}|s)$.

\textbf{Step 1: occupancy difference.}
For any policy $\phi$, we have $d^\phi=(1-\gamma)d_0+\gamma d^\phi P^\phi$, where $d_0$ is the initial-state distribution. Therefore:
\begin{align}
d^{(S)}-d^{(\varnothing)}
&= \bigl[(1-\gamma)d_0 + \gamma d^{(S)}P^{\pi^{(S)}}\bigr] - \bigl[(1-\gamma)d_0 + \gamma d^{(\varnothing)}P^{\mu}\bigr] \\
&= \gamma\bigl[d^{(S)}P^{\pi^{(S)}}-d^{(\varnothing)}P^{\mu}\bigr] \\
&= \gamma\bigl[(d^{(S)}-d^{(\varnothing)})P^{\mu}\bigr] + \gamma d^{(S)}(P^{\pi^{(S)}}-P^{\mu})
\end{align}

Taking $\ell_1$ norms and using the fact that since every row of a Markov kernel $P$ sums to 1 and $P\ge0$, Jensen's inequality yields $\|vP\|_1 \le \|v\|_1$:
\begin{equation}
\lVert d^{(S)}-d^{(\varnothing)}\rVert_1 \;\le\; \frac{\gamma}{1-\gamma}\, \bigl\|d^{(S)}(P^{\pi^{(S)}}-P^\mu)\bigr\|_1
\end{equation}

\textbf{Step 2: kernel difference.}
Since $d^{(S)}$ is a probability distribution, for any kernel $A$:
\begin{align}
\lVert d^{(S)} A\rVert_1
  &= \sum_{s'}\Bigl|\sum_{s} d^{(S)}(s)A(s'|s)\Bigr|
    \le \sum_{s} d^{(S)}(s)\,
          \|A(\cdot|s)\|_1
    \le \sup_{s}\|A(\cdot|s)\|_1
\end{align}

Using $A=P^{\pi^{(S)}}-P^{\mu}$:
\begin{equation}
\bigl\|d^{(S)}(P^{\pi^{(S)}}-P^\mu)\bigr\|_1 \le \sup_{s}\lVert P^{\pi^{(S)}}(\cdot|s)-P^\mu(\cdot|s)\rVert_1
\end{equation}

Converting the kernel difference to a policy difference:
\begin{equation}
P^{\pi^{(S)}}(\cdot|s)-P^{\mu}(\cdot|s) = \sum_{\mathbf a}P(\cdot|s,\mathbf a)\,\bigl[\pi^{(S)}(\mathbf a|s)-\mu(\mathbf a|s)\bigr]
\end{equation}
Using the triangle inequality and the fact that each $P(\cdot|s,\mathbf a)$ has $\ell_1$-norm 1:
\begin{equation}
\sup_{s}\lVert P^{\pi^{(S)}}(\cdot|s)-P^\mu(\cdot|s)\rVert_1 \le \sup_{s}\!\sum_{\mathbf a} \bigl|\pi^{(S)}(\mathbf a|s)-\mu(\mathbf a|s)\bigr|
\end{equation}

\textbf{Step 3: product–difference inequality.}
For factorized distributions $p(\mathbf a)=\prod_i p_i(a_i)$ and $q(\mathbf a)=\prod_i q_i(a_i)$, we prove by induction that:
\begin{equation}
\sum_{\mathbf a}|p(\boldsymbol{a})-q(\boldsymbol{a})| \;\le\; 2\sum_{i}\mathrm{TV}(p_i,q_i)
\end{equation}

\paragraph{Induction on $n$.} Base case $n=1$ is the definition of TV. Assume the claim holds for $n-1$ factors; take $p(\mathbf a)=p_1(a_1)\tilde p(\mathbf a_{-1})$ and $q(\mathbf a)=q_1(a_1)\tilde q(\mathbf a_{-1})$:
\begin{align}
\sum_{\mathbf a}|p-q| &= \sum_{a_1,\mathbf a_{-1}} \bigl|p_1(a_1)\tilde p(\mathbf a_{-1}) - q_1(a_1)\tilde q(\mathbf a_{-1})\bigr| \\
&\le \sum_{a_1,\mathbf a_{-1}} |p_1(a_1)|\,|\tilde p(\boldsymbol{a}_{-1})-\tilde q(\boldsymbol{a}_{-1})| + \sum_{a_1,\mathbf a_{-1}} |\tilde q(\mathbf a_{-1})|\,|p_1(a_1)-q_1(a_1)| \\
&= \sum_{\mathbf a_{-1}}|\tilde p(\boldsymbol{a}_{-1})-\tilde q(\boldsymbol{a}_{-1})| + 2\,\mathrm{TV}(p_1,q_1)
\end{align}
By the induction hypothesis, the first term is $\le 2\sum_{i=2}^n \mathrm{TV}(p_i,q_i)$, completing the proof.

\textbf{Step 4: combine and translate to $W_1$.}
Applying Step 3 to $\pi^{(S)}$ and $\mu$ (which differ only for agents in $S$):
\begin{equation}
\sup_s \sum_{\boldsymbol{a}} |\pi^{(S)}(\boldsymbol{a}|s) - \mu(\boldsymbol{a}|s)| \le 2\sum_{i \in S} \mathrm{TV}(\pi_i(\cdot|s), \mu_i(\cdot|s))
\end{equation}
Combining Steps 1-3:
\begin{equation}
\lVert d^{(S)}-d^{(\varnothing)}\rVert_1 \;\le\; \frac{2\gamma}{1-\gamma} \sum_{i\in S}\mathrm{TV}(\pi_i,\mu_i)
\end{equation}
Finally, since $\mathrm{TV}(p,q)=\frac{1}{2}\|p-q\|_1$ and $W_1=\mathrm{TV}$ under the $0$-$1$ metric (because the transport cost of moving mass between distinct points is 1):
\begin{equation}
W_1(d^{(S)}, d^{(\varnothing)}) = \frac{1}{2}\|d^{(S)} - d^{(\varnothing)}\|_1 \le \frac{\gamma}{1-\gamma}\sum_{i \in S} \mathrm{TV}(\pi_i, \mu_i)
\end{equation}
\end{proof}

%-----------------------------------------------------------------------

\subsection{Proof of Theorem~\ref{thm:value}}
\begin{proof}
We assume that $\hat{Q}$ is $\frac{2}{1-\gamma}$-Lipschitz under the $0$-$1$ metric, which holds when $\hat{Q}$ is obtained through contractive updates that preserve the Lipschitz constant of $Q^\star$.

Let $f := Q^\star - \hat{Q}$. By the triangle inequality:
\begin{align}
|V^\pi - \hat{V}^\pi|
&= \bigl|\mathbb{E}_{d^\pi,\pi}[Q^\pi] - \mathbb{E}_{d^\pi,\pi}[\hat{Q}]\bigr| \\
&= \bigl|\mathbb{E}_{d^\pi,\pi}[Q^\pi - \hat{Q}]\bigr| \\
&= \bigl|\mathbb{E}_{d^\pi,\pi}[Q^\pi - Q^\star + Q^\star - \hat{Q}]\bigr| \\
&\le \bigl|\mathbb{E}_{d^\pi,\pi}[Q^\pi - Q^\star]\bigr| + \bigl|\mathbb{E}_{d^\pi,\pi}[Q^\star - \hat{Q}]\bigr| \\
&= \underbrace{\bigl|\mathbb{E}_{d^\pi,\pi}[Q^\pi - Q^\star]\bigr|}_{(\mathrm{I})}
   + \underbrace{\bigl|\mathbb{E}_{d^\pi,\pi}[f]\bigr|}_{(\mathrm{II})}
\end{align}

\paragraph{Term (I).}
\begin{equation}
(\mathrm{I}) \le \|Q^\pi - Q^\star\|_\infty = \varepsilon_{\mathrm{Subopt}}
\end{equation}

\paragraph{Term (II).}
\begin{align}
(\mathrm{II}) &= \bigl|\mathbb{E}_{d^\pi,\pi}[f]\bigr| \\
&= \bigl|\mathbb{E}_{d^\pi,\pi}[f] - \mathbb{E}_{d^\mu,\mu}[f] + \mathbb{E}_{d^\mu,\mu}[f]\bigr| \\
&\le \bigl|\mathbb{E}_{d^\pi,\pi}[f] - \mathbb{E}_{d^\mu,\mu}[f]\bigr| + \bigl|\mathbb{E}_{d^\mu,\mu}[f]\bigr|
\end{align}

For the second term:
\begin{equation}
\bigl|\mathbb{E}_{d^\mu,\mu}[f]\bigr| = \bigl|\mathbb{E}_{d^\mu,\mu}[Q^\star - \hat{Q}]\bigr| \le \|Q^\star - \hat{Q}\|_\infty = \varepsilon_{\mathrm{FQI}}
\end{equation}

For the first term, using the Lipschitz properties of both $Q^\star$ and $\hat{Q}$:
\begin{align}
\bigl|\mathbb{E}_{d^\pi,\pi}[f] - \mathbb{E}_{d^\mu,\mu}[f]\bigr| &= \bigl|\mathbb{E}_{d^\pi,\pi}[Q^\star - \hat{Q}] - \mathbb{E}_{d^\mu,\mu}[Q^\star - \hat{Q}]\bigr| \\
&\le \bigl|\mathbb{E}_{d^\pi,\pi}[Q^\star] - \mathbb{E}_{d^\mu,\mu}[Q^\star]\bigr| + \bigl|\mathbb{E}_{d^\pi,\pi}[\hat{Q}] - \mathbb{E}_{d^\mu,\mu}[\hat{Q}]\bigr| \\
&\le \frac{2}{1-\gamma} W_1(d^\pi,d^\mu) + \frac{2}{1-\gamma} W_1(d^\pi,d^\mu) \\
&= \frac{4}{1-\gamma} W_1(d^\pi,d^\mu)
\end{align}

Define the telescoping sequence:
\begin{align}
d^{(0)} &:= d^\mu \\
d^{(k)} &:= d^{\pi^{(\{1,\ldots,k\})}} \text{ for } k = 1,\ldots,n \\
d^{(n)} &:= d^\pi
\end{align}

By the triangle inequality and Lemma~\ref{lem:linear-div}:
\begin{align}
W_1(d^\pi,d^\mu) &\le \sum_{k=1}^n W_1(d^{(k)}, d^{(k-1)}) \\
&\le \sum_{k=1}^n \frac{\gamma}{1-\gamma} \mathrm{TV}(\pi_k, \mu_k) \\
&= \frac{\gamma}{1-\gamma} \sum_{i=1}^{n}\mathrm{TV}(\pi_i,\mu_i)
\end{align}

Therefore:
\begin{equation}
(\mathrm{II}) \le \varepsilon_{\mathrm{FQI}} + \frac{4\gamma}{(1-\gamma)^2} \sum_{i=1}^{n}\mathrm{TV}(\pi_i,\mu_i)
\end{equation}

Combining terms (I) and (II):
\begin{equation}
|V^\pi - \hat{V}^\pi| \le \varepsilon_{\mathrm{Subopt}} + \varepsilon_{\mathrm{FQI}} + \frac{4\gamma}{(1-\gamma)^2} \sum_{i=1}^{n}\mathrm{TV}(\pi_i,\mu_i)
\end{equation}
\end{proof}

%-----------------------------------------------------------------------

\subsection{Proof of Proposition~\ref{prop:grad-equiv}}
\begin{proof}
Let $J_i(\theta)=\frac{1}{2}\mathbb{E}_{\mathcal{D}}\left[(Q_\theta-\mathcal{T}_i^{\mathrm{ind}}Q_\theta)^2\right]$ be the TD loss for agent $i$. 

Using the chain rule for the squared loss:
\begin{align}
\nabla_\theta J_i(\theta) &= \nabla_\theta \left[\frac{1}{2}\mathbb{E}_{\mathcal{D}}\left[(Q_\theta-\mathcal{T}_i^{\mathrm{ind}}Q_\theta)^2\right]\right] \\
&= \frac{1}{2}\mathbb{E}_{\mathcal{D}}\left[\nabla_\theta (Q_\theta-\mathcal{T}_i^{\mathrm{ind}}Q_\theta)^2\right] \\
&= \frac{1}{2}\mathbb{E}_{\mathcal{D}}\left[2(Q_\theta-\mathcal{T}_i^{\mathrm{ind}}Q_\theta) \nabla_\theta (Q_\theta-\mathcal{T}_i^{\mathrm{ind}}Q_\theta)\right] \\
&= \mathbb{E}_{\mathcal{D}}\left[(Q_\theta-\mathcal{T}_i^{\mathrm{ind}}Q_\theta) \nabla_\theta (Q_\theta-\mathcal{T}_i^{\mathrm{ind}}Q_\theta)\right]
\end{align}

Since $\mathcal{T}_i^{\mathrm{ind}}Q_\theta$ is treated as a constant target with respect to the current parameters $\theta$:
\begin{equation}
\nabla_\theta (Q_\theta-\mathcal{T}_i^{\mathrm{ind}}Q_\theta) = \nabla_\theta Q_\theta
\end{equation}

Therefore:
\begin{equation}
\nabla_\theta J_i(\theta) = \mathbb{E}_{\mathcal{D}}\left[(Q_\theta-\mathcal{T}_i^{\mathrm{ind}}Q_\theta) \nabla_\theta Q_\theta\right]
\end{equation}

Define the averaged-individual operator:
\begin{equation}
\mathcal{T}^{\mathrm{ai}}Q_\theta = \frac{1}{n}\sum_{i=1}^n \mathcal{T}_i^{\mathrm{ind}}Q_\theta
\end{equation}

The averaged objective is:
\begin{equation}
J^{\mathrm{ai}}(\theta) = \frac{1}{n}\sum_{i=1}^n J_i(\theta)
\end{equation}

Taking the gradient:
\begin{align}
\nabla_\theta J^{\mathrm{ai}}(\theta) &= \nabla_\theta \left[\frac{1}{n}\sum_{i=1}^n J_i(\theta)\right] \\
&= \frac{1}{n}\sum_{i=1}^n \nabla_\theta J_i(\theta) \\
&= \frac{1}{n}\sum_{i=1}^n \mathbb{E}_{\mathcal{D}}[(Q_\theta-\mathcal{T}_i^{\mathrm{ind}}Q_\theta) \nabla_\theta Q_\theta]
\end{align}

By linearity of expectation:
\begin{align}
\nabla_\theta J^{\mathrm{ai}}(\theta) &= \mathbb{E}_{\mathcal{D}}\left[\frac{1}{n}\sum_{i=1}^n (Q_\theta-\mathcal{T}_i^{\mathrm{ind}}Q_\theta) \nabla_\theta Q_\theta\right] \\
&= \mathbb{E}_{\mathcal{D}}\left[\left(\frac{1}{n}\sum_{i=1}^n Q_\theta - \frac{1}{n}\sum_{i=1}^n \mathcal{T}_i^{\mathrm{ind}}Q_\theta\right) \nabla_\theta Q_\theta\right] \\
&= \mathbb{E}_{\mathcal{D}}\left[\left(Q_\theta - \frac{1}{n}\sum_{i=1}^n \mathcal{T}_i^{\mathrm{ind}}Q_\theta\right) \nabla_\theta Q_\theta\right] \\
&= \mathbb{E}_{\mathcal{D}}[(Q_\theta-\mathcal{T}^{\mathrm{ai}}Q_\theta) \nabla_\theta Q_\theta]
\end{align}

This expression is precisely the gradient of the centralized TD loss:
\begin{equation}
J^{\mathrm{centralized}}(\theta) = \frac{1}{2}\mathbb{E}_{\mathcal{D}}\left[(Q_\theta-\mathcal{T}^{\mathrm{ai}}Q_\theta)^2\right]
\end{equation}

To verify, the gradient of this centralized loss is:
\begin{align}
\nabla_\theta J^{\mathrm{centralized}}(\theta) &= \mathbb{E}_{\mathcal{D}}\left[(Q_\theta-\mathcal{T}^{\mathrm{ai}}Q_\theta) \nabla_\theta Q_\theta\right]
\end{align}

which matches our derived expression.

\textbf{Note:} This gradient equivalence holds for the temporal difference component only. The CQL regularizer is applied uniformly to the shared Q-function and does not affect this equivalence.
\end{proof}

%-----------------------------------------------------------------------

\subsection{Lipschitz constant derivation}
\label{app:lipschitz}

We show that under the $0$-$1$ metric, $Q^\star$ is $\frac{2}{1-\gamma}$-Lipschitz.

Under the $0$-$1$ metric, the distance between any two distinct points is:
\begin{equation}
d((s,\boldsymbol{a}), (s',\boldsymbol{a}')) = \begin{cases}
0 & \text{if } (s,\boldsymbol{a}) = (s',\boldsymbol{a}') \\
1 & \text{if } (s,\boldsymbol{a}) \neq (s',\boldsymbol{a}')
\end{cases}
\end{equation}

For $Q^\star$ to be $L$-Lipschitz, we need:
\begin{equation}
|Q^\star(s,\boldsymbol{a}) - Q^\star(s',\boldsymbol{a}')| \le L \cdot d((s,\boldsymbol{a}), (s',\boldsymbol{a}'))
\end{equation}

When $(s,\boldsymbol{a}) = (s',\boldsymbol{a}')$, the inequality holds trivially. When $(s,\boldsymbol{a}) \neq (s',\boldsymbol{a}')$, we have $d = 1$, so we need:
\begin{equation}
|Q^\star(s,\boldsymbol{a}) - Q^\star(s',\boldsymbol{a}')| \le L
\end{equation}

Since rewards are bounded by $|R(s,\boldsymbol{a})| \le 1$, the optimal Q-function satisfies:
\begin{align}
Q^\star(s,\boldsymbol{a}) &= \mathbb{E}\left[\sum_{t=0}^\infty \gamma^t R(s_t,\boldsymbol{a}_t) \,\middle|\, s_0=s, \boldsymbol{a}_0=\boldsymbol{a}\right] \\
|Q^\star(s,\boldsymbol{a})| &\le \mathbb{E}\left[\sum_{t=0}^\infty \gamma^t |R(s_t,\boldsymbol{a}_t)|\right] \\
&\le \sum_{t=0}^\infty \gamma^t \cdot 1 = \frac{1}{1-\gamma}
\end{align}

Therefore, for any two state-action pairs:
\begin{align}
|Q^\star(s,\boldsymbol{a}) - Q^\star(s',\boldsymbol{a}')| &\le |Q^\star(s,\boldsymbol{a})| + |Q^\star(s',\boldsymbol{a}')| \\
&\le \frac{1}{1-\gamma} + \frac{1}{1-\gamma} = \frac{2}{1-\gamma}
\end{align}

Hence, $Q^\star$ is $\frac{2}{1-\gamma}$-Lipschitz under the $0$-$1$ metric.

% %%%%%%%%%%%

% \newpage
% \input{ReproducibilityChecklist}

\end{document}